%% file: main.tex
\documentclass[runningheads]{llncs}
\usepackage{wrapfig}
\usepackage[utf8]{inputenc} % allow utf-8 input
\usepackage[T1]{fontenc}    % use 8-bit T1 fonts
\usepackage{hyperref}       % hyperlinks
\usepackage{url}            % simple URL typesetting
\usepackage{booktabs}       % professional-quality tables
\usepackage{amsfonts}       % blackboard math symbols
\usepackage{nicefrac}       % compact symbols for 1/2, etc.
\usepackage{microtype} 
\usepackage{amsfonts}
\usepackage[switch]{lineno}
\usepackage{breqn}
\usepackage{graphicx}
\usepackage{makecell}
\usepackage{algorithm}
\usepackage{algorithmic}
\usepackage{caption}
\usepackage{enumitem}
\usepackage{todonotes}
\usepackage{float}
\usepackage{placeins}
\DeclareMathOperator{\Var}{Var}
\DeclareMathOperator{\Cov}{Cov}

\begin{document}
\title{Selective Pseudo-Label Clustering}
\author{Louis Mahon, Thomas Lukasiewicz}
\author{Louis Mahon \and
Thomas Lukasiewicz}

\authorrunning{L. Mahon et al.}
\institute{Department of Computer Science\\
University of Oxford, UK}
%\email{louis.mahon@linacre.ox.ac.uk}\\

\maketitle              % typeset the header of the contribution

\input{text/abstract}
\input{text/intro}
\input{text/related_work}

\input{text/method}
\input{text/proof}

\input{text/results}

\input{text/conclusion}
%\clearpage

%\bibliographystyle{splncs04}
\bibliography{bibliography}

\input{text/appendices}

\end{document}

%% file: text/abstract.tex
\begin{abstract}
Deep neural networks (DNNs) offer a means of addressing the challenging task of clustering high-dimensional data. DNNs can extract useful features, and so produce a lower dimensional representation, which is more amenable to clustering techniques. As clustering is typically performed in a purely unsupervised setting, where no training labels are available, the question then arises as to how the DNN feature extractor can be trained. The most accurate existing approaches combine the training of the DNN with the clustering objective, so that information from the clustering process can be used to update the DNN to produce better features for clustering. One problem with this approach is that these ``pseudo-labels'' produced by the clustering algorithm are noisy, and any errors that they contain will hurt the training of the DNN. In this paper, we propose selective pseudo-label clustering, which uses only the most confident pseudo-labels for training the~DNN. We formally prove the performance gains under certain conditions. Applied to the task of image clustering, the new approach achieves a state-of-the-art performance on three popular image datasets.
\end{abstract}

%% file: text/intro.tex
\section{Introduction}
Clustering is the task of  partitioning a dataset into clusters such that data points within the same cluster are similar to each other, and data points from different clusters are different to each other. It is applicable to any set of data for which there is a notion of similarity between data points. It requires no prior knowledge, neither the explicit labels of supervised learning nor the knowledge of expected symmetries and invariances leveraged in self-supervised learning. 

The result of a successful clustering is a means of describing data in terms of the cluster that they belong to. This is a ubiquitous feature of human cognition. For example, we hear a sound and think of it as an utterance of the word ``water'', or we see a video of a biomechanical motion and think of it as a jump.  This can be further refined among experts, so that a musician could describe a musical phrase as an English cadence in A major, or a dancer could describe a snippet of ballet as a right-leg fouette into arabesque. When clustering high-dimensional data, the curse of dimensionality \cite{bellman1966dynamic} means that many classic algorithms, such as k-means \cite{lloyd1957least} or expectation maximization \cite{dempster1977maximum}, perform poorly. The Euclidean distance, which is the basis for the notion of similarity in the Euclidean space, becomes weaker in higher dimensions \cite{zimek2012survey}. Several solutions to this problem have been proposed. In this paper, we consider those termed deep clustering.

%\subsection{Deep Clustering}
%Deep neural networks (DNNs) are effective feature extractors from high-dimensional data. Convolutional neural networks (CNNs) have proved especially effective at feature extraction from images. For example, \cite{krizhevsky2012imagenet} has shown that early layers in a CNN can be trained to respond to visual features such as lines, corners, and edges, and \cite{zeiler2014visualizing}  that later layers can respond to more semantic features, such as parts of objects. 
Deep clustering is a set of techniques that use a DNN to encode the high-dimensional data into a lower-dimensional feature space, and then perform clustering in this feature space. A major challenge is the training of the encoder. Much of the success of DNNs as image feature extractors (including \cite{krizhevsky2012imagenet,zeiler2014visualizing}) has been in supervised settings, but if we already had labels for our data, then there would be no need to cluster in the first place. There are two common approaches to training the encoder. The first is to use the reconstruction loss from a corresponding decoder, i.e., to train it as an autoencoder \cite{zemel1994developing}. The second is to design a clustering loss, so that the encoding and the clustering are optimized jointly. Both are discussed further in Section \ref{sec:related-work}. 

Our model, \emph{selective pseudo-label clustering} (\emph{SPC}), combines reconstruction and clustering loss. It uses an ensemble to select different loss functions for different data points, depending on how confident we are in their 
predicted~clusters. 

%\subsection{Ensembles and Consensus Clustering}
Ensemble learning is a %method of 
function approximation where multiple approximating models are trained, and then the results are combined. % in some way. 
Some variance across the ensemble is required. If all individual approximators were identical, there would be no gain in combining them. For ensembles composed of DNNs, variance is ensured by the random initializations of the weights and stochasticity of the training dynamics. In the simplest case, the output of the ensemble is the average of each individual output (mean for regression and mode for classification) \cite{opitz1997empirical}.

When applying an ensemble to clustering problems (referred to as consensus clustering; see \cite{boongoen2018cluster} for a comprehensive discussion), the sets of cluster labels must be aligned across the ensemble. This can be performed efficiently using the Hungarian algorithm. SPC considers a clustered data point to be confident if it received the same cluster label (after alignment) in each member of the ensemble. The intuition is that, due to random initializations and stochasticity of training, there is some non-zero degree of independence between the different sets of cluster labels, so the probability that all cluster labels are incorrect for a particular point is less than the probability that a single cluster label is incorrect.

Our main contributions are briefly summarized as follows.

\begin{itemize}
%[itemindent=0ex,leftmargin=2.5ex]
    \item[--] We describe a generally applicable deep clustering method (SPC), which treats cluster assignments as pseudo-labels, and introduces a novel technique to increase the accuracy of the pseudo-labels used for training. This produces a better feature extractor, and hence a more accurate clustering. 
    \item[--] We formally prove the advantages of SPC, given some simplifying assumptions. Specifically, we prove that our method does indeed increase the accuracy of the targets used for pseudo-label training, and this increase in accuracy does indeed lead to a better clustering performance. 
    \item[--] We implement SPC for  image clustering, % (code available on publication), 
    with a state-of-the-art performance on three popular image clustering datasets, and we present ablation studies on its main components. 
\end{itemize}

The rest of this paper is organized as follows. Section~\ref{sec:related-work} gives an overview of related work. Sections \ref{sec:method} and \ref{sec:proof} give a detailed description of SPC and a proof of correctness, respectively. Section \ref{sec:results} presents and discusses our experimental results, including a comparison to existing image clustering models and ablation studies on main components of SPC. Finally, Section \ref{sec:conclusion} summarizes our results and gives  an outlook on future work. Full proofs and further details are in the~appendix.

%% file: text/related_work.tex
\section{Related Work} \label{sec:related-work}
%\todo{reduce this section to 1 page max}
One of the first deep image clustering models was \cite{huang2014deep}. 
%Following, e.g.,  \cite{hinton2006reducing}, it leverages the dimensionality reduction of DNNs by training an autoencoder to reconstruct the input images, and then performs clustering in the latent space of the autoencoder. 
It trains an autoencoder (AE) on reconstruction loss (rloss), and then clusters in the latent space, using loss terms to make the latent space more amenable to clustering.

In \cite{yang2017towards}, the training of the encoder is integrated with the clustering. A~second loss function is defined as the distance of each encoding to its assigned centroid. It then alternates between updating the encoder and clustering by k-means. A~different differentiable loss is proposed in \cite{xie2016unsupervised}, based on a soft cluster assignment using Student's $t$-distribution. The method pretrains an AE on rloss, then, like \cite{yang2017towards}, alternates between assigning clusters and training the encoder on cluster loss. Two slight modifications were made in later works: use of rloss after pretraining in \cite{guo2017improved} and regularization to encourage equally-sized clusters in \cite{Dizaji_2017_ICCV}.

This alternating optimization is replaced in \cite{gao2020deep} by a clustering loss that allows cluster centroids to be optimized directly by gradient descent. 

%Pseudo-label training is introduced by 
Pseudo-label training is introduced by \cite{caron2018deep}. Cluster assignments are interpreted as pseudo-labels, which are then used to train a multilayer perceptron on top of the encoding DNN, training alternates between clustering encodings, and treating these clusters as labels to train the encoder.
%Interestingly, \cite{caron2018deep} does not train on rloss at all. The initial clustering is performed on the output of a randomly initialized encoder, and it is still possible  to use these pseudo-labels to improve future rounds of clustering. However, a quantitative comparison with other models is difficult, as the work does  not report clustering accuracy directly. The technique of using cluster labels as training labels, is referred to as pseudo-label training.

%could cut this paragraph if need space, these works not in results table
%Some methods are based on prior knowledge that certain data points should be assigned the same cluster. In \cite{haeusser2018associative}, four transformed versions of each image are generated using reflections, rotations, and brightness shifts, and then assigned to the same cluster as the original image. This teaches the clustering to be invariant to the specific transformations applied. Prior knowledge is also used in \cite{abavisani2018deep}, which trains on pairs of images at a time, both of the same digit.

Generative adversarial networks \cite{goodfellow2014generative} (GANs) have produced impressive results in image synthesis \cite{Karras_2019_CVPR,elgammal2017can,brock2018large}. At the time of writing, the most accurate GAN-based image clustering models \cite{Mukherjee2019ClusterGANL,ding2019clustering}  design a generator to sample from a latent space that is the concatenation of a multivariate normal vector and a categorical one-hot encoding vector, then recover latent vectors for the input images as in \cite{creswell2018inverting,lipton2017precise}, and cluster the latent vectors. A similar idea is employed in \cite{jiang2016variational}, though not in an adversarial setting. For more details on GAN-based clustering, see \cite{Zhou_2018_CVPR,zhao2018gan,ding2019clustering,Liang_2018_ECCV,wang2019wegan} and the references therein.

Adversarial training is used for regularization in \cite{mrabah2019adversarial}. In \cite{mrabah2019deep}, the method is developed. Conflicted data points are identified as those whose maximum probability across all clusters is less than some threshold, or whose max and next-to-max are within some threshold of each other. Pseudo-label training is then performed on the unconflicted points only. A similar threshold-based filtering method is employed by \cite{chang2017deep}. 

A final model to consider is \cite{mcconville2019n2d}, which uses a second round (i.e., after the DNN) of dimensionality reduction via UMAP \cite{mcinnes2018umap}, before clustering.
%Many of the above regularizations of the latent space aim to encourage some form of locality. UMAP  preserves local structure, so it is expected to take the place of such regularizations, also shown empirically in~\cite{mcconville2019n2d}. 

%% file: text/method.tex
\section{Method} \label{sec:method}
 Pseudo-label training is an effective deep clustering method, but training on only partially accurate pseudo-labels can hurt the encoder's ability to extract relevant features. Selective pseudo-label clustering (SPC) addresses this problem by selecting only the most confident pseudo-labels for training, using the four steps shown in~Fig.~\ref{fig:method-description}. 

\begin{figure*}[t]
    \centering
    \includegraphics[width=\textwidth]{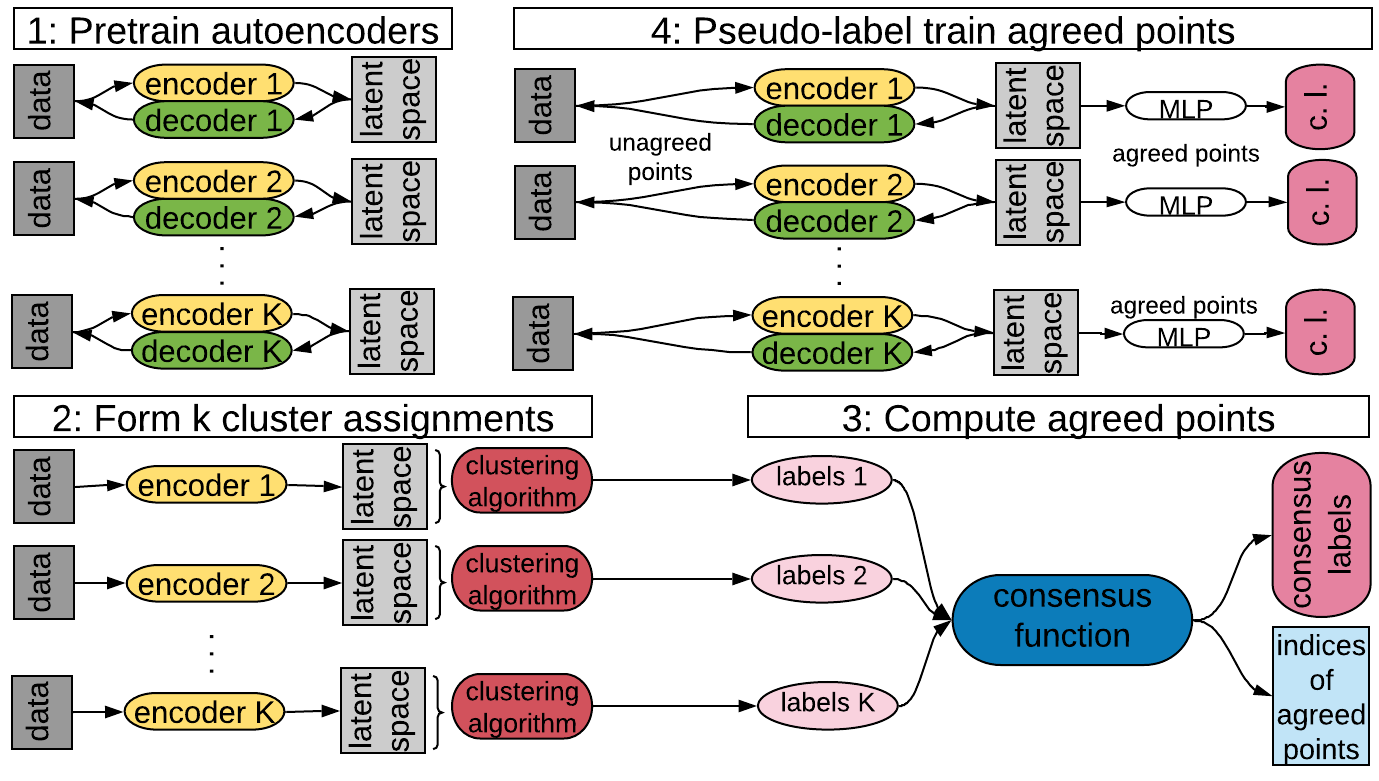}
    \caption{ \small The complete SPC method. (1) Pretrain autoencoders. (2)~Perform multiple clusterings independently. (3) Identify agreed points as those that receive the same label in all ensemble members. (4) Perform pseudo-label training on agreed points and autoencoder training on unagreed points. Steps (2)--(4) are looped until the number of agreed points stops increasing.}
    \label{fig:method-description}
    \vspace{-3ex}
\end{figure*}

\begin{enumerate}
%[itemindent=0ex,leftmargin=2.5ex]
    \item Train $K$ autoencoders in parallel.
    \item Cluster in the latent space of each, to obtain~$K$ sets of pseudo-labels.
    \item Select for pseudo-label training, those points are those that received the same label in all $K$ sets of pseudo-labels, after the labellings have been aligned using the Hungarian algorithm.
    \item Train on the selected pseudo-labels. Go back to (2).
\end{enumerate} 
Training ends when the number of agreed points stops increasing. Then, each data point is assigned  its most common cluster label across the (aligned) ensemble.

\subsection{Formal Description} \label{subsec:formal-description}
Given a dataset $\mathcal{X} \subseteq \mathbb{R}^n$ of size $N$ with $C$ true clusters, let $(f_j)_{1 \leq j \leq K},  f_j:\mathbb{R}^n \rightarrow \mathbb{R}^m$, and $(g_j)_{1 \leq j \leq K}, g_j:  \mathbb{R}^m \rightarrow \mathbb{R}^n$ be the $K$ encoders and decoders, respectively. 
%Let
%\begin{equation} \label{eq:r_loss}
%\mathcal{L}_r = \frac{1}{N}\sum_{i=1}^N\sum_{j=1}^K  ||g_j(f_j(x_i)) - x_i|| \,.
%\end{equation}
Let $\psi: \mathbb{R}^{N \times m} \rightarrow \{0,\dots,C-1\}^N$ be the clustering function, which takes the $N$ encoded data points as input, and returns a cluster label for each. We refer to the output of $\psi$ as a labelling. Let $\Gamma: \{0,\dots,C-1\}^{K \times N} \rightarrow \{0,\dots,C-1\}^N \times \{0,1\}^N$ be the consensus function, which aggregates $K$ different labellings of $\mathcal{X}$ into a single labelling, and also returns a Boolean vector indicating agreement. Then, 
\begin{equation} \label{eq:ensemble}
{\small(c_1, \dots, c_N), (a_1, \dots, a_N) = \Gamma(\psi(f_1(\mathcal{X})) \circ \cdots \circ \psi(f_K(\mathcal{X}))),}
\end{equation}
where $(c_1, \dots, c_N)$ are the consensus labels, and $a_i\,{ =}\, 1$ if the $i$-th data point received the same cluster label (after alignment) in all labellings, and~$0$ otherwise. The consensus function is the ensemble mode average, $c_i$ is the cluster label that was most commonly assigned to the $i$-th data point. 

Define $K$ pseudo-classifiers \mbox{$(h_j)_{1 \leq j \leq K}, h_j: \mathbb{R}^m \rightarrow \mathbb{R}^C$}, and let 
\begin{gather} \label{eq:combined}
\mathcal{L} = \frac{1}{N}\sum_{i=1}^N\sum_{j=1}^K 
\begin{cases} 
CE(h_j(f_j(x_i)), c_i) & a_i = 1 \\
||g_j(f_j(x_i)) - x_i|| & \text{otherwise,}
\end{cases} 
\end{gather}
where $CE$ denotes categorical cross-entropy:
\begin{gather*}
CE: \mathbb{R}^C \times \{0, \dots, C-1\} \rightarrow \mathbb{R} \\
CE(x,n) = -\log (x[n]) \,.
\end{gather*}
%Note that, because we only pseudo-label train on agreed points, training each encoder on the consensus labels $c_1, \dots, c_N$ is equivalent to training each encoder on its own labels. For agreed points, all ensemble labels, and hence the consensus labels, are the same up to reordering. 

\begin{algorithm}[t]
\begin{algorithmic}
\caption {Training algorithm for SPC} \label{alg:method}
\FOR {$j=1,\dots, K$}
    \STATE Update parameters of $f_j$ and $g_j$ using autoencoder reconstruction
\ENDFOR
\WHILE {number of agreed points increases}
    \STATE compute $(c_1, \dots, c_N), (a_1, \dots, a_N)$ as in \eqref{eq:ensemble}
    \FOR {$j=1,\dots, K$}
        \STATE Update parameters of $f_j$ and $h_j$ to minimize \eqref{eq:combined}
    \ENDFOR
\ENDWHILE
\end{algorithmic}
\end{algorithm}

\noindent First, we pretrain the autoencoders, then compute $(c_1, \dots, c_N), (a_1, \dots, a_N)$ and minimize $\mathcal{L}$, recompute, and iterate until the number of agreed points stops increasing. The method is summarized in Algorithm \ref{alg:method}.

Figure \ref{fig:training-dynamics} shows the training dynamics. Agreed points are those that receive the same cluster label in all members of the ensemble. As expected, the agreed points' accuracy is higher than the accuracy on all points. Initially, the agreed points will not include those that are difficult to cluster correctly, such as an MNIST digit 3 that looks like a digit 5. Some ensemble members will cluster it as a 3 and others as a 5. The training process aims to make these difficult points into agreed points, thus increasing the fraction of agreed points, without decreasing the agreed points' accuracy. Figure \ref{fig:training-dynamics} shows that this aim is achieved. As more points become agreed (black dotted line), the total accuracy approaches the agreed accuracy. The agreed accuracy remains high, decreasing only very slightly (blue line). The result is that the total accuracy increases (orange line). We end training when the number of agreed points plateaus. 

\begin{figure}[t]
    %{r}{0.50\linewidth}
    \centering
    %\vspace*{-2ex}
    %\vspace{-45pt}
    \includegraphics[width=.65\textwidth]{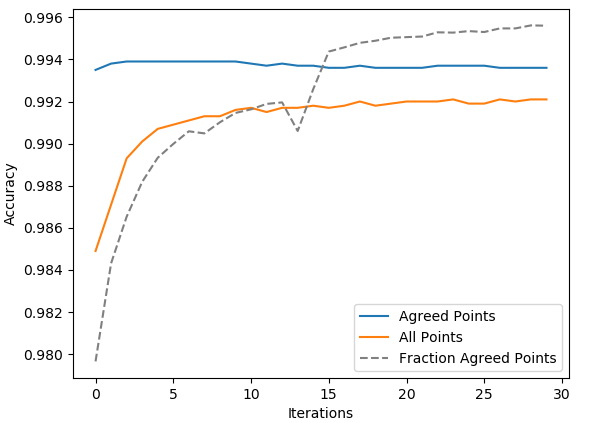}\vspace*{-1ex}
    \caption{ \footnotesize Iterations of (2)--(4) in Figure \ref{fig:method-description} on MNIST.}
    \label{fig:training-dynamics}\vspace*{-2ex}
    %\vspace{-25pt}
\end{figure}

\subsection{Implementation Details}
Encoders are stacks of convolutional and batch norm layers; decoders of transpose convolutional layers. Decoders have a $tanh$ activation on their output layer, all other layers use leaky ReLU. The MLP pseudo-classifier has a hidden layer of size $25$. The latent space of the autoencoders has the size $50$ for MNIST and FashionMNIST, and $20$ for smaller USPS. We inject noise from a multivariate normal into the latent space as a simple form of regularization. As suggested in \cite{zhao2015loss}, the reconstruction loss is $\ell_1$. The architectures are the same across the ensemble, diversity comes from random initialization and training dynamics.

The clustering function ($\psi$ above) is a composition of UMAP \cite{mcinnes2018umap} and either HDBSCAN \cite{mcinnes2017hdbscan} or a Gaussian mixture model (GMM). As in previous works, we set the number of clusters to the ground truth. UMAP uses the parameters suggested in the clustering documentation clustering, $n\_neighbours$ is $30$ for MNIST and scaled in proportion to the dataset size for the others. HDBSCAN uses all default parameters. We cut the linkage tree at a level that gives the correct number of clusters. On the rare occasions when no such cut can be found, the clustering is excluded from the ensemble. The GMM uses all default parameters.

Consensus labels are taken as the most common across the ensemble, after alignment with the Hungarian algorithm (called the ``direct'' method in \cite{boongoen2018cluster}).

%% file: text/proof.tex
\section{Proof of Correctness} \label{sec:proof}
Proving correctness requires proving that the expected accuracy of the agreed pseudo-labels is higher than that of all pseudo-labels, and that training with more accurate pseudo-labels makes the latent vectors easier to cluster correctly.

\subsection{Agreed Pseudo-Labels are More Accurate}
Given that each member of the ensemble is initialized independently at random, and undergoes different stochastic training dynamics, we can assume that each cluster assignment contains some unique information. Formally, there is strictly positive conditional mutual information between any one assignment in the ensemble and the true cluster labels, conditioned on all the other assignments in the ensemble. From this assumption, the reasoning proceeds as follows.

Choose an arbitrary data point $x_0$ and cluster $c_0$. Let $X$ be a random variable (r.v.), indicating the true cluster of $x_0$, given that $n$ members of the ensemble have assigned it to $c_0$, and other assignments are unknown, $n\geq 0$. Thus, the event $X=c_0$ is the event that $x_0$ is correctly clustered. Let $Y$ be a Boolean r.v.\ indicating that the \mbox{$(n+1)$-th} member of the ensemble also assigns it to $c_0$. Assume that, if $n$ ensemble members have assigned $x_0$ to $c_0$, and other assignments are unknown, then $x_0$ belongs to $c_0$ with probability at least $1/C$ and belongs to all other clusters with equal probability, i.e.,  
\begin{gather*}
    p(X=c_0) = t \\
    \forall c \neq c_0, p(X=c) = {(1-t)}/{(C-1)}\,,
\end{gather*}
for some $1/C \leq t \leq 1$. 
It follows that the entropy $H(X)$ is a strictly decreasing function of $t$ (see appendix for proof). 
Thus, the above assumption on conditional mutual information, written $I(X;Y) > 0$, is equivalent to $p(X\,{=}\,c_0 | Y) > p(X\,{=}\,c_0)$. This establishes that the accuracy of the agreed labels is an increasing function of ensemble size. Standard pseudo-label training uses $n\,{=}\,1$, whereas SPC uses $n>1$ and so results in more accurate pseudo-labels for training.

\subsection{Increased Pseudo-Label Accuracy Improves Clustering}
%Although it is intuitive that training on more accurate pseudo-labels makes the clustering task easier, and although it has been implicitly assumed in previous works using pseudo-label training \cite{caron2018deep, mrabah2019deep}, a formal proof has not been provided. We prove the claim in a simplified case. This validates both the general strategy of pseudo-label training, and our strategy of training on only the most accurate pseudo-labels. 

\subsubsection{Problem Formulation.}
Let $\mathcal{D}$ be a dataset of i.i.d.~points from a distribution over $\mathcal{S} \in \mathbb{R}^n$, where $\mathcal{S}$ contains $C$ true clusters $c_1, \dots, c_C$. Let $T$ be the r.v.defined by the identity function on~$S$ and $f:\mathcal{S} \rightarrow \mathbb{R}^m$, an encoding function parametrized by $\theta$, whose output is an r.v. $X$. The task is to recover the true clusters conditional on $X$, and we are interested in choosing $\theta$ such that this task is as easy as possible. Pseudo-label training applies a second function $h: \mathbb{R}^m \rightarrow \{0, \dots, C-1\}$ and trains the composition $h \circ f: \mathbb{R}^n \rightarrow  \{0, \dots, C-1\}$ using gradient descent (g.d.), with cluster assignments as pseudo-labels. The claim is that an increased pseudo-label accuracy facilitates a better choice of $\theta$. 

 To formalize ``easy'', recall the definition of clustering as a partition that minimizes intra-cluster variance and maximizes inter-cluster variance. We want the same property to hold of the r.v. $X$. Let $y:\mathcal{D} \rightarrow \{0, \dots, C-1\}$ be the true cluster assignment function and $Y$ the corresponding random variable, then ease of recovering the true clusters is captured by a high value of $d$, where
\[
d =  \Var(\mathbb{E}[X|Y]) - \mathbb{E}[\Var(X|Y)] \,.
\]
High $d$ means that a large fraction of the variance of $X$ is accounted for by cluster assignment, as, by Eve's law, we can decompose:
\begin{equation} \label{eve}
\Var{(X)} = \mathbb{E}[\Var(X|Y)] + \Var(\mathbb{E}[X|Y]) \,,
\end{equation}

In the following, we assume that $f$ and $g$ are linear, $C=2$, $h \circ f (\mathcal{D}) \subseteq (0,1)$, and $\mathbb{E}[T] = \vec{0}$. The proof proceeds by expressing the value of $d$ in terms of expected distances between encoded points after a training step with correct labels and with incorrect labels, and hence proving that the value is greater in the former case. We  show that the expectation is greater in each coordinate, from which the claim follows by linearity (see appendix for details).

\begin{lemma} \label{lemma:smaller-dist-if-same}
    Let $x,  x' \in \mathcal{D}$ be two data points, and consider the expected squared distance between their encodings under~$f$. Let $u_{\mathit{same}}$ and $u_{\mathit{diff}}$ denote the value of this difference after a g.d. update in which both labels are the same and after a step in which both labels are different, respectively. Then, $u_{\mathit{same}} < u_{\mathit{diff}}$. 
\end{lemma}

If $w \in \mathbb{R}^{m}$ and $w' \in \mathbb{R}$ are, respectively, the vector of weights mapping the input to the $i$-th coordinate of the latent space, and the scalar mapping the $i$-th coordinate of the latent space to the output, then the expected squared distance in the $i$-th coordinate of the latent vectors before the g.d. update is 
\[
\underset{x, x' \sim T}{\mathbb{E}}[(w^Tx - w^Tx')^2] = \underset{x, x' \sim T}{\mathbb{E}}[(w^T(x - x'))^2] \,.
\]
When the two labels are the same, assume w.l.o.g.~that $y = y' = 0$. Then, with step size $\eta$ , the update for $w$ and following expected squared difference $u_{same}$ is 
\begin{align*} 
&w \gets w - \eta (w'(x+x')) \\
u_{\mathit{same}} =& \underset{x, x' \sim T}{\mathbb{E}}[((w - \eta w'(x+x'))^T(x - x'))^2] \\
=& \underset{x, x' \sim T}{\mathbb{E}}[(w^T(x - x') - \eta w'(||x||^2-||x'||^2))^2]  \,.
\end{align*}
When the two labels are different, assume w.l.o.g.~that $y = 0$, $y' = 1$, giving
\begin{gather*} 
w \gets w - \eta (w'(x-x')) \\
u_{\mathit{diff}} = \underset{x, x' \sim T}{\mathbb{E}}[((w - \eta (w'(x-x'))^T)(x - x')])^2] = \\
\underset{x, x' \sim T}{\mathbb{E}}[(w^T(x - x') - \eta w'||x-x'||^2)^2]  \,.
\end{gather*} 
It can then be shown (see appendix) that $u_{\mathit{same}} < u_{\mathit{diff}}$.

\begin{lemma} \label{lemma:equal-distance-from-3rd}
    Let $z$ be a third data point, $z \in \mathcal{D}, z \neq x, x'$, and consider the expected squared distance of the encodings, under $f$, of $x$ and $z$. Let $v_{\mathit{same}}$ and $v_{\mathit{diff}}$ denote, respectively, the value of this difference after a g.d. update with two of the same labels, and with two different labels. Then, $v_{\mathit{same}} = v_{\mathit{diff}}$.
\end{lemma}

\begin{lemma} \label{lemma:same-diff-decomp}
    Let $s$ and $r$ denote, respectively, the expected squared distance between the encodings, under $f$, of two points in the same cluster and between two points in different clusters. Then, there exist $\lambda_1, \lambda_2 > 0$ whose values do not depend on the parameters of $f$, such that $d = \lambda_1r - \lambda_2s$.
\end{lemma}
For simplicity, assume that the clusters are equally sized. The argument can easily be generalized to clusters of arbitrary sizes. We then obtain 
\[
d = \frac{C-1}{2C}r - \frac{2C -1}{2C}s\,,
\]
where $C$ is the number of clusters (see appendix for proof).

\begin{definition}{\rm 
    Let $\tilde{y}:\mathcal{D} \rightarrow \{0, \dots, C-1\}$ be the pseudo-label assignment function. For  $d_i, d_j \in \mathcal{D}$, the pseudo-labels are \emph{pairwise correct} iff $y(x_i) = y(x_j)$ and $\tilde{y}(x_i) = \tilde{y}(x_j)$, or $y(x_i) \neq y(x_j)$ and $\tilde{y}(x_i) \neq \tilde{y}(x_j)$.}
\end{definition}

%Note that more pairwise correct labels corresponds to a more accurate clustering, as measured with the Rand Index. 

\begin{theorem}
    Let $d_{T}$ and $d_{F}$ denote, respectively, the value of $d$ after a g.d. step from two pairwise correct labels and from two pairwise incorrect labels, and let $x,x' \in \mathcal{D}$ as before. Then, $d_T > d_F$.
\end{theorem}
\begin{proof}
Let $r_T, s_T$, and $r_F, s_F$ be, respectively, the values of $r$ and $s$ after a g.d. step from two pairwise correct labels and from two pairwise incorrect labels.
Consider two cases. If~$y(x) = y(x')$, then $r_T = r_F$, by Lemma \ref{lemma:equal-distance-from-3rd}, and $s_T < s_F$, by Lemmas \ref{lemma:smaller-dist-if-same} and \ref{lemma:equal-distance-from-3rd}, so by Lemma \ref{lemma:same-diff-decomp}, $d_T > d_F$. If $y(x) \neq y(x')$, then $s_T = s_F$, by Lemma \ref{lemma:equal-distance-from-3rd}, and $r_T > r_F$, by Lemmas \ref{lemma:smaller-dist-if-same} and \ref{lemma:equal-distance-from-3rd}, so again $d_T > d_F$, by Lemma \ref{lemma:same-diff-decomp}.
\end{proof}
The fraction of pairwise correct pairs is one measure of accuracy (Rand Index). Thus, training with more accurate pseudo-labels facilitates better clustering.

%% file: text/results.tex
  \section{Experimental Results} \label{sec:results}
\noindent Following previous works, we measure accuracy and normalized mutual information (NMI). Accuracy is computed by aligning the predicted cluster labels with the ground-truth labels using the Hungarian algorithm \cite{kuhn1955hungarian} and then calculating as in the supervised case. NMI, as in \cite{witten2002data}, is defined as 
${2I(\tilde{Y};Y)}/({H(\tilde{Y}) + H(Y)})$, 
where $\tilde{Y}$, $Y$, $I(\cdot,\cdot)$, and $H(\cdot)$ are, respectively, the cluster labels, ground truth labels, mutual information, and Shannon entropy. We report on two handwritten digits datasets, MNIST (size 70000) \cite{lecun1998gradient} and USPS (size 9298) \cite{hull1994database}, and FashionMNIST (size 70000) \cite{xiao2017fashion} of clothing items. Table \ref{tab:main-results} shows the central tendency for five runs and the best single run.

\begin{table*}[th]
%    \vspace{-4.5ex}
%  \medskip
%  \centering
%  \footnotesize
 \centerline{\scalebox{0.91}{\begin{tabular}{@{}lllllll@{}}
    \toprule
    & \multicolumn{2}{c}{MNIST} & \multicolumn{2}{c}{USPS} & \multicolumn{2}{c}{FashionMNIST}        \\
    \cmidrule(r){2-3}  \cmidrule(r){4-5}  \cmidrule(r){6-7}
    & ACC & NMI & ACC & NMI & ACC & NMI \\
    
    \midrule
    
    SPC-best & \textbf{99.21} & \textbf{97.49} & \textbf{98.44} & \textbf{95.44} & \textit{67.94} & \textbf{73.48} \\ 
    
    SPC & \textbf{99.03 (.1)} & \textbf{97.04 (.25)} & \textbf{98.40 (.94)} & \textbf{95.42 (.15)} & 65.58 (2.09) & \textbf{72.09 (1.28)} \\ 
    
    SPC-GMM & \textbf{99.05 (.2)} & \textbf{97.10 (.47)} & \textbf{98.18 (.14)} & \textbf{94.93 (.32)} & 65.03 (1.54) & \textbf{69.51 (1.21)} \\ 
    
   \midrule
   
   DynAE \cite{mrabah2019deep}$\dagger$ & \textit{98.7} & \textit{96.4} & \textit{98.1} & 94.8 & 59.1 & 64.2 \\
    
    ADC \cite{mrabah2019adversarial}$\dagger$  & 98.6 & 96.1 & \textit{98.1} & 94.8 & 58.6 & 66.2\\
    
    DDC \cite{ren2020deep}$\dagger$ & 98.5 & 96.1 & 97.0 & \textit{95.3} & 57.0 & 63.2\\
    
    n2d \cite{mcconville2019n2d} & 97.9 & 94.2 & 95.8 & 90.0 &67.2 & \textit{68.4}\\
    
    DLS \cite{ding2019clustering}  & 97.5 & 93.6 & - & - & \textbf{69.3} & 66.9 \\
      
    JULE \cite{yang2016joint} & 96.4 & 91.3 & 95.0 & 91.3 & 56.3* & 60.8*\\
    
    DEPICT \cite{Dizaji_2017_ICCV}  & 96.5 & 91.7 & 96.4 & 92.7 & 39.2* & 39.2*\\
    
    DMSC \cite{abavisani2018deep}$\dagger$  & 95.15 & 92.09 & 95.15 & 92.09 & - & -\\
    
    ClusterGAN \cite{Mukherjee2019ClusterGANL}  & 95 & 89 & - & - & 63 & 64\\
    
    VADE \cite{jiang2016variational}  & 94.5 & 87.6* & 56.6* & 51.2* & 57.8* & 63.0*\\
    
    IDEC \cite{guo2017improved}  & 88.06 & 86.72 & 76.05 & 78.46 & 52.9* & 55.7*\\
    
    CKM \cite{gao2020deep}  & 85.4 & 81.4 & 72.1 & 70.7 & - & -\\
    
    DEC \cite{xie2016unsupervised}  & 84.3 & 83.4* & 76.2* & 76.7* & 51.8* & 54.6*\\
    
    DCN \cite{yang2017towards}  & 83 & 81 & 68.8* & 68.3* & 50.1* & 55.8*\\
    
   \bottomrule
  \end{tabular}}}  
  \vspace{1ex}
  \centerline{$\dagger = $uses data augmentation \hspace{10pt} *=results taken from \cite{mrabah2019deep}}
  \vspace{1.5ex}
\caption{\footnotesize Accuracy and NMI of SPC compared to other top-performing image clustering models. The best results are in bold, and the second-best are emphasized. We report the mean and standard deviation (in parentheses) for five runs. } \label{tab:main-results}
  \vspace{-4ex}
  \end{table*}

We show results for two different clustering algorithms: Gaussian mixture model and the more advanced HDBSCAN \cite{mcinnes2017hdbscan}. Both perform similarly, showing robustness to clustering algorithm choice. SPC-GMM performs slightly worse on USPS and FashionMNIST (though within margin of error), suggesting that HDBSCAN  may cope better with the more complex images in FashionMNIST and the smaller dataset in USPS. In Table \ref{tab:main-results}, `SPC' uses HDBSCAN.

SPC (using either clustering algorithm) outperforms all existing approaches for both metrics on MNIST and USPS, and for NMI on FashionMNIST. The disparity between the two metrics, and between HDBSCAN and GMM, on FashionMNIST is due to the variance in cluster size. Many of the errors are lumped into one large cluster, and this hurts accuracy more than NMI, because being in this large cluster still conveys some information about what the ground truth cluster label is (see appendix for full details).

The most accurate existing methods use data augmentation. This is to be expected, given the well-established success of data augmentation in supervised learning \cite{hinton2012improving}. More specifically, \cite{guo2018deep} have shown empirically that adding data augmentation to deep image clustering models improves performance in virtually all cases. Here, its effect is especially evident on the smaller dataset, USPS. For example, on MNIST, n2d \cite{mcconville2019n2d} (which does not use data augmentation) is only $0.6$ and $1.9$ behind DDC \cite{ren2020deep}, which does on ACC and NMI, respectively, but is $1.2$ and $5.3$ behind on USPS. SPC could easily be extended to include data augmentation, and even without using it, outperforms models that do.

\subsection{Ablation Studies}
\begin{table*}[!t]
    %\vspace{-4.5ex}
   \medskip 
  \centering
  \footnotesize
  \begin{tabular}{@{}lllllll@{}}
    \toprule
    & \multicolumn{2}{c}{MNIST} & \multicolumn{2}{c}{USPS} & \multicolumn{2}{c}{FashionMNIST}        \\
    \cmidrule(r){2-3}  \cmidrule(r){4-5}  \cmidrule(r){6-7}
    & ACC & NMI & ACC & NMI & ACC & NMI \\
    \midrule
    
    SPC& \textbf{99.03 (.1)} & \textbf{97.04 (.25)} & \textbf{98.40 (.94)} & \textbf{95.42 (.15)} & \textbf{65.58 (2.09)} & \textbf{72.09 (1.28)} \\ 
    
    \midrule
   
    A1 & 98.01 (.04) & 94.46 (.11) & 97.03 (.65) & 92.43 (1.29) & 63.12 (.16) & 70.59 (.01) \\
    
    A2 & 98.18 (.05) & 94.86 (.09) & 97.31 (.89) & 92.99 (1.84) & 60.60 (4.45) & 68.77 (.48) \\
    
    A3 & 98.02 (.19) & 94.45 (.43) & 95.85 (.80) & 89.77 (1.65) & 59.23 (3.58) & 67.09 (3.77) \\
    
    A4 & 97.88 (.72) & 94.8 (.85) & 87.49 (7.93) & 82.68 (2.6) & 61.2 (4.28) & 67.28 (1.72) \\
    
    A5 & 96.17 (.26) & 91.07 (.23) & 87.00 (8.88) & 80.79 (7.43) & 55.29 (3.54) & 66.07 (1.04) \\
    
    A6 & 70.24 & 77.42 & 70.46 & 71.11 & 42.08 & 49.22 \\
     
   \bottomrule
  \end{tabular}
  \vspace{2ex}
  \caption{ \footnotesize Ablation results, central tendency for three runs. A1=w/o label filtering; A2=w/o label sharing; A3=w/o ensemble; A4=pseudo-label training only; A5=UMAP+AE; A6=UMAP. Both A1 and A2 train on all data points. The former trains each member of the ensemble on their own labels, and the latter uses the consensus labels. A3 sets $K=1$, in the notation of Section \ref{subsec:formal-description}.} \label{tab:ablation-results}
 \vspace{-4ex}
\end{table*}

\begin{table*}[t]%\vspace*{1ex}
\centering
  \footnotesize
\vspace*{-0.5ex}
 \centering
  \footnotesize
  \begin{tabular}{@{}lllllll@{}}
    \toprule
    & \multicolumn{2}{c}{MNIST} & \multicolumn{2}{c}{USPS} & \multicolumn{2}{c}{FashionMNIST}        \\
    \cmidrule(r){2-3}  \cmidrule(r){4-5}  \cmidrule(r){6-7}
    & ACC & NMI & ACC & NMI & ACC & NMI \\
    \midrule
    25 & 98.48 & 95.60 & 97.70 & 93.82 & 67.67 & 73.25 \\
    20 & 98.49 & 95.64 & 97.87 & 94.21 & 67.52 & 73.13 \\
    15 & 99.03 (.10) & 97.04 (.25) & 98.40 (.94) & 95.42 (.15) & 65.58 (2.09) & 72.09 (1.28) \\
    12 & 98.82 & 96.54 & 98.20 & 95.02 & 67.77 & 73.13 \\
    10 & 98.78 & 96.42 & 98.39 & 95.47 & 62.93 & 69.89 \\
    8 & 98.75& 96.32 & 98.41 & 95.44 & 67.45 & 71.99 \\
    6 & 98.61 & 95.90 & 98.40 & 95.39 & 63.84 & 70.62 \\
    5 & 98.56 & 95.82 & 98.30 & 95.19 & 67.91 & 73.46 \\
    4 & 98.47 & 95.60 & 98.27 & 95.18 & 67.90 & 73.38 \\
    3 & 98.44 & 95.50 & 98.15 & 94.84 & 63.36 & 70.88 \\
    2 & 98.27 & 95.07 & 97.98 & 94.40 & 62.9 & 70.41 \\
    1 & 98.02 (.19) & 94.45 (.43) & 95.85 (.80) & 89.77 (1.65) & 59.23 (3.58) & 67.09 (3.77) \\
    \hline
    \end{tabular}\vspace*{2ex}
\caption{\footnotesize Ablation studies on the size of the ensemble.}
\vspace*{-5ex}
\label{tab:ensemble-size-ablation}\end{table*}

\noindent Table \ref{tab:ablation-results} shows the effect of removing each component of our model. All settings use HDBSCAN. Particularly relevant are rows 2 and 3. As described in Section~\ref{sec:method}, we produce multiple labellings of the dataset and select for pseudo-label training only those data points that received the same label in all labellings. We perform two different ablations on this method: A1 and A2. Both use all data points for training, but A1 trains each ensemble on all data points using the labels computed in that ensemble member, and A2 uses the consensus labels. At inference, both use consensus labels. The significant drop in accuracy in both settings demonstrates that the strong performance of SPC is not just due to the application of an ensemble to existing methods, but rather to the novel method of label selection. 

It is interesting to observe that A1 performs worse than A2 on MNIST and USPS. Combining approximations in an ensemble has long been observed to give higher expected accuracy (\cite{clemen1989combining,pearlmutter1991chaitin,perrone1993improving,breiman1996bagging}), so the training targets would be more accurate in A1 than in A2. We hypothesize that the reason that this fails to translate to improved clustering is a reduction in ensemble variance. On MNIST and USPS, high accuracy across the ensemble means high agreement. Giving the same training signal for every data point reduces variance further. Especially, compared with A2,  the reduction is greatest on incorrectly clustered data points, because most incorrectly clustered data points are non-agreed points, and as argued in \cite{kittler1998combining}, high ensemble variance in the errors is important for performance. 

A4 clusters in the latent space of one untrained encoder and then pseudo-label trains (essentially the method  in \cite{caron2018deep}). It performs significantly worse than SPC, showing the value of the decoder, and of SPC's label selection technique.

A3 omits the ensemble entirely. Comparing with A2 again shows that the ensemble itself only produces a small improvement. Alongside SPC's label selection method, the improvement is much greater.

\subsection{Ensemble Size}
 The number of autoencoders in the ensemble, $K$ in the terminology of Section \ref{subsec:formal-description}, is a hyperparameter. We add the concatenation of all latent spaces as an additional element. Table~\ref{tab:ensemble-size-ablation} shows the performance for smaller ensemble sizes. In MNIST and USPS, where the variance is reasonably small, there is a discernible trend of the performance increasing with $K$, then plateauing and starting to decrease. For FashionMNIST, where the variance is higher, the pattern is less clear. For all three datasets, however, we can see a significant difference between an ensemble of size two and an ensemble of size one (i.e., no ensemble). We hypothesize that the decrease for $K=20,25$ is due to a decrease in the number of agreed points, and so fewer pseudo-labels to train the encoders. 
 
%The main results in Table \ref{tab:main-results} use $K\,{=}\,15$, and the trend in Table \ref{tab:ensemble-size-ablation} suggests that the  performance is unlikely to improve significantly beyond this number. Indeed, for USPS and FashionMNIST, it had already plateaued for smaller values of $K$. 

%% file: text/conclusion.tex
\section{Conclusion} \label{sec:conclusion}
This paper has presented a deep clustering model, called selective pseudo-label clustering (SPC). SPC employs pseudo-label training, which alternates between clustering features extracted by a DNN, and treating these clusters as labels to train the DNN. We have improved this framework  with a novel technique for preventing the DNN from learning noise. The method is formally sound and achieves a state-of-the-art performance on three popular image clustering datasets. Ablation studies have demonstrated that the high accuracy is not merely the result of applying an ensemble to existing techniques, but rather is due to SPC's novel filtering method. Future work includes the application to other clustering domains, different from images, and an investigation of how SPC combines with existing deep clustering techniques. 

\medskip 
\noindent{\textbf{Acknowledgments. }}%
This work was supported by the Alan Turing Institute
under the UK EPSRC grant EP/N510129/1 and by the AXA
Research Fund.  We also acknowledge the use of
the EPSRC-funded Tier 2 facility JADE (EP/P020275/1) and
GPU computing support by Scan Computers International Ltd.

%% file: text/appendices.tex
\newpage
\section{Appendix A: Full Proofs}
This appendix contains the full proofs of the results in Section~\ref{sec:proof}.

\input{text/full_proofs}

\section{Appendix C: Extended Results} \label{app:extended-results}

The results in the main paper report the central tendency of five different training runs for each dataset. Tables \ref{tab:mnist-cluster-sizes}, \ref{tab:usps-cluster-sizes}, and \ref{tab:fash-cluster-sizes} show the sizes of the clusters predicted by SPC for one randomly selected run out of these five. On MNIST and USPS, where the accuracy of SPC is $>98\%$, the predicted sizes are close to the true sizes. On FashionMNIST, where the accuracy is $\sim 65\%$, there is a much greater variance. This accounts for the discrepancy in ACC and NMI for FashionMNIST. Most of the errors are put into one large cluster, specifically the cluster that was aligned to `coat' is over three times larger than it should be. This hurts accuracy more than NMI, because the incorrect data points in the `coat' cluster count for zero when calculating the accuracy, but they are not randomly distributed among the other classes, so the conditional entropy of a data point that was mis-clustered as a coat is $<\log(10)$. Actually, most of the mistakes in the `coat' cluster are pullovers or shirts, and almost none of them are, for examples, boots or tops. Comparing the cluster sizes for SPC-HDBSCAN and SPC-GMM also accounts for the differences across ACC and NMI between these two settings on FashionMNIST: SPC-GMM produces more uniformly-sized clusters, so the difference between ACC and NMI is smaller.

%% file: text/full_proofs.tex
\subsection{More Accurate Pseudo-Labels Supplement}
The only part omitted from the argument in the main paper is a proof for the claim about the entropy of the random variable~$X$. This is supplied by the following proposition.
\begin{proposition}
Given a categorical random variable $X$ of the form 
\begin{gather*}
    p(X=c_0) = t \\
    \forall c \neq c_0, p(X=c) = \frac{1-t}{C-1}\,,
\end{gather*}
for some $1/C \leq t \leq 1$, the entropy $H(X)$ is a strictly decreasing function of $t$. 
\end{proposition}
\begin{proof}
\begin{align*}
    H(X) =& - t\log t - (1-t)\log \frac{1-t}{C-1} \\
    \frac{d(H(X))}{dt} =& - \log t - 1 - \frac{1}{1-t} + \log \frac{1}{C-1} + \\
    & + \frac{t}{1-t} + \log 1-t \\
    =& -2 - \log t - \log C-1 + \log t-1 \\
    =& -2 - \log \left(\frac{t}{1-t}(C-1)\right).
\end{align*}
The argument to the $\log$ is clearly an increasing function of $t$ for $t>1$. Therefore, for $1/C \leq t < 1$, it is lower-bounded by setting $t = 1/C$. This gives
\begin{align*}
    \frac{d(H(X))}{dt} &\leq -2 - \log \left(\frac{1/C}{1-1/C}(C-1)\right) \\
    &< -\log \left(\frac{1/C}{1-1/C}(C-1)\right) = -\log 1 = 0\,.
\end{align*}
The derivative is always strictly negative with respect to $t$, so, as a function of $t$, $H(X)$ is always strictly decreasing.
\end{proof}

\subsection{Lemma 1 Supplement}
The following is a proof for the claim that $u_{\mathit{same}} < u_{\mathit{diff}}$, as stated in Section \ref{sec:proof}.

%rewrite u_same
Decomposing $u_{\mathit{same}}$ according to the definition of variance (as the expectation of the square minus the square of the expectation) gives 
\begin{align*}
\underset{x, x' \sim T}{\mathbb{E}}[w^T(x - x') - \eta w'(||x||^2-||x'||^2)]^2 + \\ \text{Var}(w^T(x - x') - \eta w'(||x||^2-||x'||^2))\,.
\end{align*}
The expectation term equals $0$, as
\begin{gather*}
w^T\underset{x, x' \sim T}{\mathbb{E}}[(x - x')] - \eta w'\underset{x, x' \sim T}{\mathbb{E}}[(||x||^2-||x'||^2)] = \\
(w\mathbb{E}[T] - \mathbb{E}[T]) - \eta w'(\mathbb{E}[||T||^2]-\mathbb{E}[||T||^2]) = 0 \,.
\end{gather*}
By symmetry, we can replace covariances involving $x'$ with the same involving $x$. The remaining term can then be rearranged to give
\begin{gather*}
u_{same} = 2\text{Var}(w^Tx - \eta w'||x||^2) \\
 = 2w^TCov(T)w + 2\eta w' Var(||x||^2) - 4\text{Cov}(w^Tx, \eta w'||x||^2)\,.
\end{gather*}

%rewrite u_diff
Now rewrite $u_{\mathit{diff}}$. Decomposing as above gives
\begin{gather*}
\underset{x, x' \sim T}{\mathbb{E}}[w^T(x - x') - \eta w'(||x-x'||^2)]^2 + \\ \text{Var}(w^T(x - x') - \eta w'(||x-x'||^2))\,,
\end{gather*}
and here the expectation term does not equal 0:
\begin{gather*}
(w^T\underset{x, x' \sim T}{\mathbb{E}}[(x - x')] - \eta w'\underset{x, x' \sim T}{\mathbb{E}}[(||x-x'||^2)])^2 = \\
(\eta w')^2\underset{x, x' \sim T}{\mathbb{E}}[||x-x'||^2]^2
\,.
\end{gather*}
The variance term can be expanded to give:
\begin{gather*}
    \text{Var}(w^T(x - x') - \eta w'(||x-x'||^2)) = \\
    2w^TCov(T)w + 2\eta w'\text{Var}(||x - x'||^2) - \\
    4\text{Cov}(w^Tx, \eta w'||x - x'||^2)\,.
\end{gather*}

%show covariance terms are equal
By comparing terms, we can see that this expression is at least as large as  $u_{same}$. First, consider the covariance terms. 
\[
\textbf{Claim.} \hspace{10pt} \text{Cov}(w^Tx, \eta w'||x - x'||^2) = \text{Cov}(w^Tx,\eta w'||x||^2). \\[10pt]
\]
\begin{align*}
    &\text{Cov}(w^Tx, \eta w'||x - x'||^2)  \\
    &=\mathbb{E}[w^Tx\eta w'||x-x'||^2] - \mathbb{E}[w^Tx]\mathbb{E}[\eta w'||x - x'||^2] \\
    &= \eta w'\mathbb{E}[w^Tx||x-x'||^2] - 0\mathbb{E}[\eta w'||x - x'||^2]\\
    &= \eta w'\mathbb{E}[w^Tx||x-x'||^2] \\
    &= \eta w'\mathbb{E}[w^Tx\sum_{k}x^2 - 2xx' + x^{'2}] \\
    &= \eta w'\sum_{k}\mathbb{E}[w^Txx_k^2] - 2\mathbb{E}[w^Txx_k]\mathbb{E}[x'] + \mathbb{E}[w^Tx]\mathbb{E}[x^{'2}] \\
    &= \eta w'\sum_{k}\mathbb{E}[w^Txx_k^2] - 2\mathbb{E}[w^Txx_k]\vec{0} + 0\mathbb{E}[x^{'2}] \\
    &= \eta w'\sum_{k}\mathbb{E}[w^Txx_k^2] \\
    &= \eta w'\mathbb{E}[w^Tx\sum_{k}x_k^2] \\
    &= \eta w'\mathbb{E}[w^Tx||x||^2] \\
    &= \eta w'\mathbb{E}[w^Tx||x||^2] - 0\mathbb{E}[\eta w'||x||^2] \\
        &= \mathbb{E}[w^Tx\eta w'||x||^2] - \mathbb{E}[w^Tx]\mathbb{E}[\eta w'||x||^2] \\
    &= \text{Cov}(w^Tx,\eta w'||x||^2)  \,.
\end{align*}
So, we see the covariance terms are equal. 

%show second variance term greater in u_diff
Next, compare the second variance terms
\[
\textbf{Claim.} \ \text{Var}(||x - x'||^2) \geq \text{Var}(||x||^2). \\[10pt]
\]
\begin{align*}
    &\text{Var}(||x - x'||^2)  \\
    =& \text{Var}\left(\sum_{k=0}^{nz}(x)_k^2  + (x')_k^2 - 2(x)_k(x')_k\right) \\
    =\ & \text{Var}\left(\sum_{k=0}^{nz}(x)_k^2\right) + \text{Var}\left(\sum_{k=0}^{nz}(x')_k^2\right) + 2\Var \left(\sum_{k=0}^{nz} x_kx'_k\right) \\
    =\ & 2\text{Var}\left(\sum_{k=0}^{nz}(x)_k^2\right)  + 2\Var \left(\sum_{k=0}^{nz} x_kx'_k\right)\\
    =\ & 2(\text{Var}(||x||^2) + \Var(x^Tx'))\\
    \geq \ & \text{Var}(||x||^2) \,.
\end{align*}
Assuming that the data are not all identical, this implies that $u_{\mathit{diff}}$ is strictly greater than $u_{same}$.
\begin{align*}
    &u_{\mathit{diff}} - u_{same}  \\
    = \ &(\eta w')^2\underset{x, x' \sim T}{\mathbb{E}}[||x-x'||^2]^2 + 2w^T\Cov(T)w + \\ &+2\eta w'\text{Var}(||x - x'||^2) - 4\Cov(w^Tx, \eta w'||x - x'||^2) - \\
        &- ((2w^T\Cov(T)w + 2\eta w' Var(||x||^2) \\
     &   - 4\Cov(w^Tx, \eta w'||x||^2))) \\
    = \ & (\eta w')^2\underset{x, x' \sim T}{\mathbb{E}}[||x-x'||^2]^2 + \\
    \hspace{10pt}& + 2\eta w'\left(\text{Var}(||x - x'||^2) - \text{Var}(||x||^2)\right) - \\
    & - 4\left(\Cov(w^Tx, \eta w'||x - x'||^2) - \Cov(w^Tx, \eta w'||x||^2)\right) \\
    = \ & (\eta w')^2\underset{x, x' \sim T}{\mathbb{E}}[||x-x'||^2]^2 + \\
    \hspace{10pt}& + 2\eta w'\left(\text{Var}(||x - x'||^2) - \text{Var}(||x||^2)\right) \\
    \geq \ & (\eta w')^2\underset{x, x' \sim T}{\mathbb{E}}[||x-x'||^2]^2 > 0\,.
\end{align*}

\subsection{Lemma 2 Supplement}\vspace*{-1ex}
The following is the complete proof of Lemma 2, which was omitted from the main paper.

\begin{proof} $v_{\mathit{diff}} - v_{same}$
\begin{align*}
    %& =\\
    =\,& \mathbb{E}[(w^T(x-z) - w'(x-x')(x-z))^2] - \\
    &-\mathbb{E}[(w^T(x-z) - w'(x+x')(x-z))^2] \\
    =\ & \mathbb{E}[(w^T(x-z) - w'(x-x')^T(x-z))^2 - \\
    &-(w^T(x-z) - w'(x+x')^T(x-z))^2] \\
    =\ & \mathbb{E}[(w^T(x-z) - w'(x-x')^T(x-z) + \\
    &+w^T(x-z) - w'(x+x')^T(x-z)) \\
        &(w^T(x-z) - w'(x-x')^T(x-z) - \\
        &-w^T(x-z) - w'(x+x')^T(x-z))] \\
    =\ & \mathbb{E}[(2w^T(x-z) - w'(x-z)^T(x-x'+x+x'))\\
    &(- w'(x-z)^T(x-x'-x-x'))] \\
    =\ & \mathbb{E}[(2w^T(x-z) - 2w'(x-z)^T(x))(2w'(x-z)^T(x'))] \\
    =\ & 2\mathbb{E}[(w^T(x-z) - w'(x-z)^T(x))w'(x-z)^T]\mathbb{E}[x'] \\
    =\ & 2\mathbb{E}[(w^T(x-z) - w'(x-z)(x))w'(x-z)^T]\vec{0} = 0\,.
\end{align*}
\end{proof}

\subsection{Lemma 3 Supplement}
The following is the complete proof of Lemma 3, which was omitted from the main paper.

\begin{proof}
\begin{align*}
    \Var(T) =& \tfrac{1}{2}\underset{x, x' \sim T}{\mathbb{E}}[(x-x')^2] \\
    =& \tfrac{1}{2}(\underset{x, x' \sim T}{\mathbb{E}}[(x-x')^2|y(x) = y(x')]P(y(x) = y(x')) + \\ &\underset{x, x' \sim T}{\mathbb{E}}[(x-x')^2|y(x) \neq y(x')]P(y(x) \neq y(x'))\\
    =& \tfrac{1}{2}(sP(y(x) = y(x')) +  rP(y(x) \neq y(x'))\\
    =& \frac{1}{2}\left(s\frac{1}{C} +  r\frac{C-1}{C}\right)\,.
\end{align*}
Noting that $s = 2\mathbb{E}[\Var(T|C)]$, and using Eve's law, we have
\begin{align*}
    d =& \Var(T) -s \\
    =& \frac{1}{2}\left(s\frac{1}{C} +  r\frac{C-1}{C}\right) -s \\
    =& \frac{C-1}{2C}r - \frac{2C -1}{2C}s\,.
\end{align*}
\end{proof}

\begin{table*}[t]
\centering
  \footnotesize
    \medskip 
 \begin{tabular}{@{}lllllllllll@{}}
    \toprule
& \thead{Zero} & \thead{One} & \thead{Two} & \thead{Three} & \thead{Four} & \thead{Five} & \thead{Six}  & \thead{Seven} & \thead{Eight} & \thead{Nine}\\
\hline
    HDBSCAN & 6923 & 7878 & 6979 & 7095 & 6802 & 6290 & 6911 & 7384 & 6776 & 6962\\ 
    \hline
    GMM & 6942 & 6958 & 6791 & 7885 & 6976 & 7096 & 7350 & 6294 & 6906 & 6802\\ 
    \hline
    Ground Truth & 7000 & 7000 & 7000 & 7000 & 7000 & 7000 & 7000 & 7000 & 7000 & 7000 \\
    \hline
    \end{tabular}\smallskip 
\caption{Sizes of predicted clusters for MNIST.} \label{tab:mnist-cluster-sizes}
%\end{table*} 
%\bigskip 
%\begin{table*}[t]
\centering
  \footnotesize
    \medskip 
 \begin{tabular}{@{}lllllllllll@{}}
    \toprule
& \thead{Zero} & \thead{One} & \thead{Two} & \thead{Three} & \thead{Four} & \thead{Five} & \thead{Six}  & \thead{Seven} & \thead{Eight} & \thead{Nine}\\
\hline
    HDBSCAN & 1565 & 1272 & 933 & 819 & 856 & 706 & 833 & 787 & 693 & 834\\ 
    \hline
    GMM & 1271 & 834 & 785 & 833 & 690 & 835 & 862 & 930 & 699 & 1559\\ 
    \hline
    Ground Truth & 1553 & 1269 & 929 & 824 & 852 & 716 & 834 & 792 & 708 & 821 \\
    \hline
    \end{tabular}\smallskip 
\caption{Sizes of predicted clusters for USPS.} \label{tab:usps-cluster-sizes}
%\end{table*} 
%\bigskip 
%\begin{table*}[b]
\centering
  \footnotesize
    \medskip 
 \begin{tabular}{@{}lllllllllll@{}}
    \toprule
& \thead{Top} & \thead{Trouser} & \thead{Pullover} & \thead{Dress} & \thead{Coat} & \thead{Sandal} & \thead{Shirt}  & \thead{Sneaker} & \thead{Bag} & \thead{Boot}\\
\hline
    HDBSCAN & 7411 & 6755 & 56 & 6591 & 21333 & 6046 & 3173 & 5666 & 3711 & 9258\\ 
    \hline
    GMM & 6700 & 3111 & 16379 & 6807 & 6753 & 9127 & 7389 & 4482 & 8814 & 438\\ 
    \hline
    Ground Truth & 7000 & 7000 & 7000 & 7000 & 7000 & 7000 & 7000 & 7000 & 7000 & 7000 \\
    \hline
    \end{tabular}\smallskip 
\caption{Sizes of predicted clusters for FashionMNIST.} \label{tab:fash-cluster-sizes}\vspace*{-3ex}
\end{table*}

\subsection{Theorem 5 Supplement}
The following is a more detailed version of the argument given in the main paper.

If $y(x) = y(x')$, then Lemma \ref{lemma:equal-distance-from-3rd} means that the expected distance of the encodings of $x$ and $x'$ to any data point from another cluster is unchanged by whether the update was from points with the same or with different labels. Similarly, the distance between any two other points is unchanged by whether the update was from points with the same or with different labels. This establishes that $r_T = r_F$. As for the intra-cluster variance, it is smaller after the update with the same labels than with different labels. Lemma \ref{lemma:smaller-dist-if-same} shows that the expected distance between the encodings of the two points themselves is smaller if the labels were the same, and the same argument as above shows that all other expected distances within clusters are unchanged.

If $y(x) \neq y(x')$, then Lemma \ref{lemma:equal-distance-from-3rd} means that the expected distance of the encodings of $x$ and any data point from the same cluster is unchanged by whether the update was from points with the same or with different labels (and the same for $x'$). Similarly, the distance between any two other points is unchanged by whether the update was from points with the same or with different labels. This establishes that $s_T = s_F$. As for the \emph{inter}-cluster variance, it is larger after the update with the same labels than with different labels. Lemma \ref{lemma:smaller-dist-if-same} shows that the expected distance between the encodings of the two points themselves is larger if the labels were different, and the same argument as above shows that all other expected distances within clusters are unchanged.